\pdfoutput=1

\documentclass[11pt]{article}

\usepackage{acl}

\usepackage{times}
\usepackage{latexsym}

\usepackage[T1]{fontenc}

\usepackage[utf8]{inputenc}

\usepackage{microtype}

\usepackage[ruled,vlined]{algorithm2e}
\usepackage{tabularx}
\usepackage{makecell}
\usepackage{algpseudocode}
\usepackage{booktabs}
\usepackage{adjustbox}
\usepackage{hhline}
\usepackage{xspace}
\usepackage{lipsum}
\usepackage{lmodern}

\newcommand\blfootnote[1]{%
  \begingroup
  \renewcommand\thefootnote{}\footnote{#1}%
  \addtocounter{footnote}{-1}%
  \endgroup
}

\makeatletter
\newcommand{\algmargin}{\the\ALG@thistlm}
\makeatother
\newlength{\whilewidth}
\algdef{SE}[parWHILE]{parWhile}{EndparWhile}[1]
  {\parbox[t]{\dimexpr\linewidth-\algmargin}{%
     \hangindent\whilewidth\strut\algorithmicwhile\ #1\ \algorithmicdo\strut}}{\algorithmicend\ \algorithmicwhile}%
\algnewcommand{\parState}[1]{\State%
  \parbox[t]{\dimexpr\linewidth-\algmargin}{\strut #1\strut}}

\usepackage{graphicx}

\usepackage{amsmath,amsthm,amsfonts,bm}
\usepackage{hyperref}

\newtheorem{theorem}{Theorem} 

\def\figref#1{figure~\ref{#1}}

\def\tblref#1{table~\ref{#1}}

\def\secref#1{section~\ref{#1}}

\def\eqref#1{equation~\ref{#1}}

\def\Algref#1{Algorithm~\ref{#1}}

\def\1{\bm{1}}

\def\ve{{\bm{e}}}

\def\vx{{\bm{x}}}

\DeclareMathAlphabet{\mathsfit}{\encodingdefault}{\sfdefault}{m}{sl}
\SetMathAlphabet{\mathsfit}{bold}{\encodingdefault}{\sfdefault}{bx}{n}

\DeclareMathOperator*{\argmax}{arg\,max}

\author{Ruohong Zhang  \quad Yau-Shian Wang \quad Yiming Yang \\
    Carnegie Mellon University \\
{\tt ruohongz@andrew.cmu.edu king6101@gmail.com} \\
{\tt yiming@cs.cmu.edu} \\
\texttt{code: \url{https://github.com/RifleZhang/GenCo}}
}

\newcommand{\ourmodel}{\textsc{GenCo}\xspace}

\usepackage{xcolor}
\newcommand{\high}[1]{\textcolor{red}{#1}}

\definecolor{ruohong}{RGB}{101, 66, 243}
\definecolor{yiming}{RGB}{255, 0, 0}

\begin{document}
%
\title{Generation-driven Contrastive Self-training for Zero-shot Text Classification with Instruction-following LLM}
%
%
%
\maketitle              

\begin{abstract}
The remarkable performance of large language models (LLMs) in zero-shot language understanding has garnered significant attention.
However, employing LLMs for large-scale inference or domain-specific fine-tuning requires immense computational resources due to their substantial model size.
To overcome these limitations, we introduce a novel method, namely \ourmodel, which leverages the strong generative power of LLMs to assist in training a smaller and more adaptable language model. 
In our method, an LLM plays an important role in the self-training loop of a smaller model in two important ways. Firstly, the LLM is used to augment each input instance with a variety of possible continuations, enriching its semantic context for better understanding. Secondly, it helps crafting additional high-quality training pairs, by rewriting input texts conditioned on predicted labels. This ensures the generated texts are highly relevant to the predicted labels, alleviating the prediction error during pseudo-labeling, while reducing the dependency on large volumes of unlabeled text.
In our experiments, \ourmodel outperforms previous state-of-the-art methods when only limited ($<5\%$ of original) in-domain text data is available. Notably, our approach surpasses the performance of Alpaca-7B with human prompts, highlighting the potential of leveraging LLM for self-training. \blfootnote{Paper accepted to EACL 2024.}

\end{abstract}
\section{Introduction}
Zero-shot text classification poses a challenge in predicting class labels for text instances without requiring labeled instances for supervised training. Effective solutions to this problem is crucial for many real-world applications, as it diminishes the labor-intensive process of manual labeling. With the remarkable advancements of large language models (LLMs)~\cite{brown2020language,ouyang2022training} in recent years, exploiting the generative capabilities of such models to tackle zero-shot text classification problems has emerged as a critical research question.

Recent research in zero-shot text classification primarily falls into two distinct groups. The first approach applies LLM (with billions of parameters) in label prediction with the help of human instructions or prompts~\cite{ouyang2022training,vicuna2023}. However, even a relatively smaller LLM such as Alpaca-7B~\cite{alpaca} necessitate considerable computational power and time for large-scale inference and model fine-tuning. Without domain-specific fine-tuning, LLMs struggle to discern between classes characterized by unclear decision boundaries. 

The second approach to zero-shot classification involves the self-training of smaller language models, often comparable in size to BERT~\cite{meng2020text, schick2020s,gera2022zero, wang2023pesco}. In these methods, the models predict "pseudo labels" for unlabeled instances, and then use these instances alongside their assigned pseudo labels as supervised data for model fine-tuning. This process is iterated for the model to incrementally adapt to the target domain. However, these techniques hinge on accessing a substantial volume of unlabeled texts from the intended domain, sometimes reaching the magnitude of millions as indicated in \tblref{tab:data_stats}, a volume that may not always be feasible in many practical contexts. Furthermore, due to the capacity limitation of small language models, the pseudo label predictions are prone to error potentially jeopardizing the efficacy of the self-training loops.

In this paper, we introduce a novel approach called \textbf{Gen}eration-driven \textbf{Co}ntrastive Self-Training (\ourmodel). This approach adeptly combines the language understanding ability of LLMs with the adaptability and efficiency of smaller models. Drawing inspiration from PESCO~\cite{wang2023pesco}, we treat zero-shot classification as a sentence alignment task and employ contrastive self-training with smaller models. We provide a theoretical analysis of how self-training can bolster classification generalization. Crucially, we sidestep the dependency on extensive unlabeled texts by capitalizing on the generative strengths of LLMs.


Our approach exploits the LLM generation power in two ways. 
Firstly, to enhance pseudo label prediction, we employ an LLM to generate multiple variations or extensions of an input text. This augmentation strategy enriches the available information for the classifier, enabling it to make better predictions based on a more comprehensive understanding of the input. 
Secondly, we employ the LLM to craft new training instances conditioned on the pseudo labels, ensuring the generated content is closely aligned with its assigned pseudo label. This tackles the prevalent issue of mislabeling in self-training. In summary, this paper makes three key contributions:
\begin{itemize}
 \item We propose a novel approach that enables smaller models to acquire knowledge from LLMs within the self-training loop. Our method is compatible with any new LLMs to effectively train better classifier on target domains. In our experiments, our small model outperforms Alpaca with human instructions.
\item We explore the more challenging setting of zero-shot classification where only a limited number of unlabeled texts are available. In this setting, we improve the performance over strong baselines.
\item We provide theoretical proof to support the effectiveness of the proposed contrastive loss for self-training.
\end{itemize}

\begin{figure*}[t!]
    \centering
    \includegraphics[width=\linewidth]{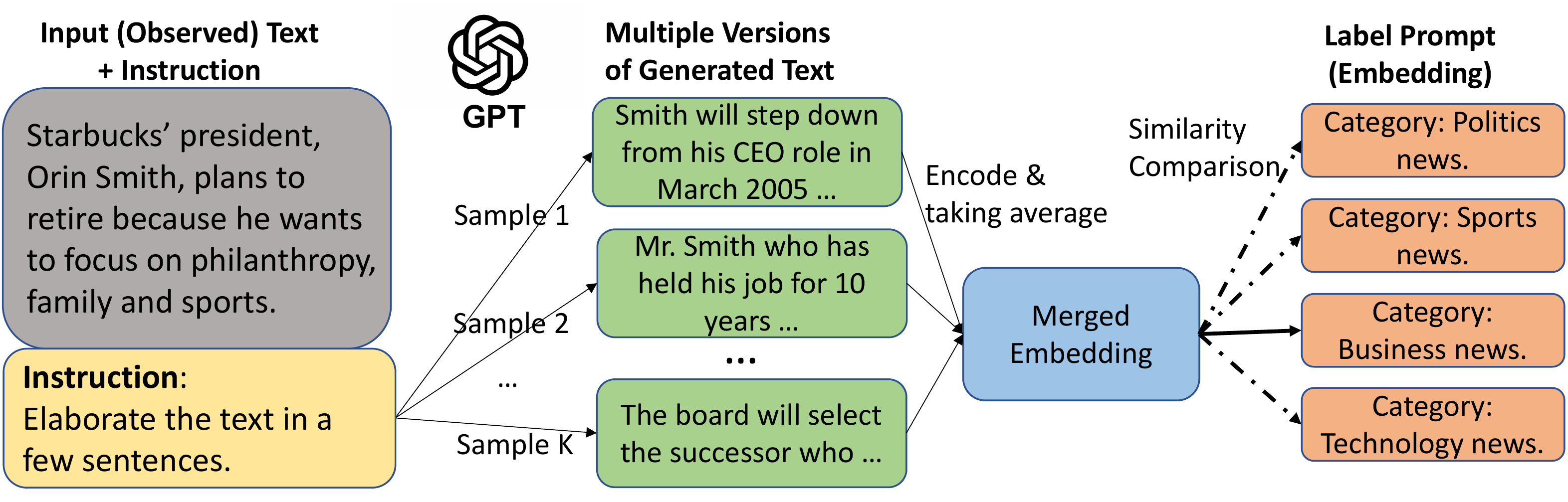}
    \caption{Enriching textual semantics through LLM Generation: The input text and an instruction are fed into the LLM to generate multiple pieces of elaborated texts, each of which is concatenated to the original input to obtain an augmented text. The embeddings of the augmented texts are then averaged to obtain a merged embedding, which is used for label prediction and contrastive loss in the self-training process. }
    \label{fig:general_aug_illustration}
\end{figure*}

\section{Preliminary: Zero-shot Text Classification as Sentence Alignment}
\label{sec:method}
Given a set of $N$ unlabeled documents $X=\{x_1,x_2,\cdots,x_N \}$ and a set of $L$ category descriptions $C=\{c_1,c_2,\cdots,c_L \}$, the goal is to learn a scoring function $g(x,c_i)$ that takes document $x$ and label description $c_i$ as input and produces a similarity score as the measure of how well the document and the label match to each other. 

In the zero-shot setting, text classification can be formulated as a sentence alignment problem~\cite{wang2023pesco}, where both the input sentence and the label descriptions are encoded using a pre-trained sentence encoder like SimCSE~\cite{gao2021simcse}. The similarity scores between the sentence and label embeddings are used to predict related labels. The performance can be further improved by converting a short label description into a full sentence via prompts~\cite{wang2023pesco, hong2022tess}. For example, the label ``sports" can be converted to ``This is an article about sports."
Subsequently, we represent the label prompt for a label description $c_i$ as $p_i$. The scoring function can be implemented as follows:
\begin{equation}
    g(x, c_i)=\operatorname{sim}\left(f_\theta(x), f_\theta(p_i) \right)
\end{equation}
where $f_\theta(\cdot)$ is the sentence encoder parameterized by $\theta$ and $\mathrm{sim}(\cdot,\cdot)$ is a similarity function such as dot product or cosine similarity.

Given an input text at inference time, the predicted  label is the one with the highest similarity score:
\begin{equation}
\hat{y}=\argmax_j g\left(x, c_j\right) \label{eq:inference}
\end{equation}

\section{Our Method: \ourmodel}
\ourmodel is a self-training framework~\cite{meng2020text,schick2021self,wang2023pesco} that harnesses the generative power of LLMs to train a smaller pre-trained sentence encoder in an iterative manner.
Each self-training step consists of two parts. First, we apply \eqref{eq:inference} to predict pseudo labels for unlabeled instances. Second, we fine-tune model on pseudo-labeled data with a proposed contrastive self-training objective. In \secref{subsec:input_augment} and \ref{subsec:condition_aug}, we will introduce two types of augmentation with LLM to enhance the self-training process.

\subsection{Contrastive Self-Training Objective}
\label{subsec:self-training}
One well-known challenge of self-training is its tendency to exhibit overconfidence in certain labels due to the model inductive bias~\cite{xie2016unsupervised}.
Extensive research has shown that soft labeling~\cite{xie2016unsupervised,meng2020text}, label smoothing~\cite{muller2019does}, and entropy regularization~\cite{grandvalet2004semi} can effectively tackle this issue.
Motivated by these, we propose to incorporate soft-labeling and entropy regularization into a contrastive loss. 

Given an input text $x$, the distribution of the predicted label space is:
\begin{equation}
\fontsize{10pt}{12pt}
    P(\hat{y}_i | x ;\theta) =  \frac{\exp(\operatorname{sim}(f_\theta(x), f_\theta(p_i)))}{\sum_{c\in C} \exp(\operatorname{sim}( f_\theta(x), f_\theta(p_c)))} 
\end{equation}
Here, $\hat{y}_i$ is the predicted label and $p_i$ is a label prompt for the predicted label. To prevent the model from being overconfident, we define the weights of the labels as:
\begin{equation}
\fontsize{10pt}{12pt}
    Q(\hat{y}_i | x ;\theta) =   \frac{\exp(\operatorname{sim}(f_\theta(x), f_\theta(p_i))/\tau)}{\sum_{c\in C} \exp(\operatorname{sim}( f_\theta(x), f_\theta(p_c))/\tau)} \label{eq:soft-label}
\end{equation},
where $\tau \leq 1$ is the temperature. A lower temperature implies a sharper distribution and thus greater weights in the predicted label. We drop the notation of $\theta$ for convenience.

Combining the above $P(\hat{y}_i | x)$ and $Q(\hat{y}_i | x )$, we propose a text to label ($t2l$) contrastive loss:
\begin{equation}
    \mathcal{L}_{t2l} = -\sum_{i=1}^{N} \sum_{j=1}^{L} Q(\hat{y}_j|x_i)\log P(\hat{y}_j|x_i) \label{eq:loss}
\end{equation}
When $\tau \rightarrow 0$, $Q(\hat{y} | x)$ becomes categorical distribution and the loss reduces to a supervised contrastive learning loss~\cite{NEURIPS2020_d89a66c7} with pseudo label $\hat{y}$ as the target:
\begin{equation}
    \mathcal{L}_{t2l}^{\tau \rightarrow 0} = -\sum_{i=1}^{N} \log P(\hat{y}|x_i)
\end{equation}
It encourages the model to predict label $\hat{y}$ given $x$ with more confident. 
On the other hand, when $\tau = 1$, the loss reduces to a minimization of conditional entropy function $H$:
\begin{align}
    &\mathcal{L}_{t2l}^{\tau = 1} = H\left(C \mid X \right) \\
    &= - \sum_{i=1}^N\sum_{j=1}^L P(\hat{y}_j|x_i)\log P(\hat{y}_j|x_i)
\end{align}
We show a theorem such that minimizing the loss function \eqref{eq:loss} can achieve similar effects Entropy Regularization~\cite{grandvalet2006entropy,grandvalet2004semi}, which is a means to enforce the cluster assumption such that the decision boundary should lie in low-density regions to improve generalization performance~\cite{chapelle2005semi}. 
\begin{theorem}
\label{theorem:paper}
Consider a binary classification problem with linearly separable labeled examples. When $0<\tau<1$, optimizing \eqref{eq:loss} with gradient descend will enforce the larger margin between classes and achieves max margin classifier under certain constraint. 
\end{theorem}
We place our formal theorems and proofs in Appendix \ref{apd:proofs}. Theorem \ref{theorem:1}  suggests that self-training is an in-domain fine-tuning that maximizes class separation, which serves as an explanation of why training on pseudo labels can enhance performance even if no extra labeling information is provided.
In our experiment, we show that self-training of a smaller model can outperform LLM (Alpaca-7B) prediction, justifying the claim empirically. We set $\tau = 0.1$ (refer to Appendix \ref{apd:temperature}) to balance supervised classification and low density separation between classes.


\begin{figure*}[t!]
    \centering
    \includegraphics[width=\linewidth]{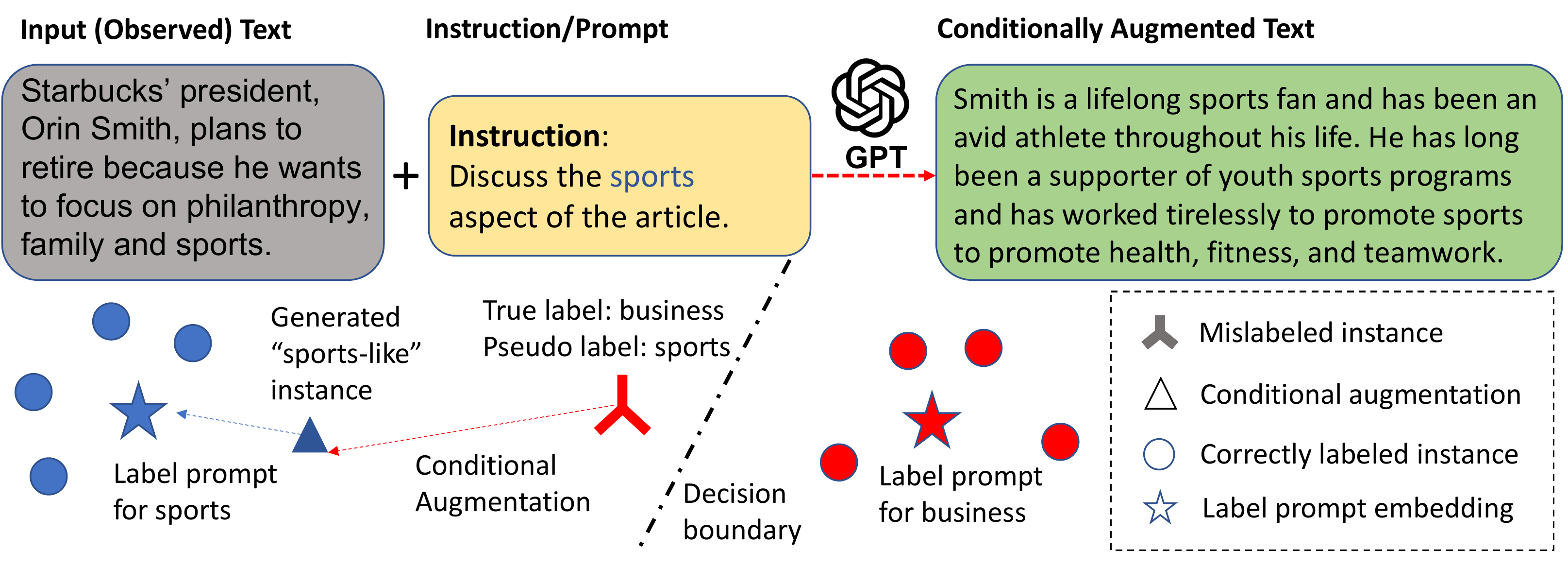}
    \caption{Conditional text augmentation to address mislabeling in self-training: When a pseudo label is incorrect, it can mislead the training process and decrease classification performance. We generate augmented text conditioned on the pseudo label, aiming to make the generated text closer to the majority members in the category of the pseudo label. This approach aims to improve the quality of the generated instances for self-training. }
    \label{fig:conditional_aug_illustration}
\end{figure*}

While self-training can potentially improve model generalization, the limitations are obvious: 1) pseudo labels are prone to error and may negatively affect model training. 2) self-learning requires a significant amount of unlabeled data, which may not always be available. To alleviate the above problems, we introduce generation-driven approaches to improve self-training with an instruction-following LLM, such as Alpaca-7B.

\subsection{Semantic Enrichment using LLM}
\label{subsec:input_augment}
In this section, we propose a way to enrich the semantic information of an input text with multiple LLM-generated pieces of text. When the input text is relatively short, such as consisting of only one or a few sentences, the information may not be sufficient for alignment-based method to match relevant labels. 

A remedy is to query an LLM to elaborate the input and generate multiple pieces of extended texts. As shown in \figref{fig:general_aug_illustration}, the instruction, "Elaborate the text with a few sentences," steers the LLM towards creating relevant expansions and continuations for the input text $x$. These augmented texts, denoted as $x^{\text{aug}}$, serve for two purposes: 1) improving the quality of pseudo label, and 2) forming the positive pair in contrastive learning, as detailed below:



\paragraph{Enhancing pseudo label quality.} We enhance pseudo label prediction by enriching the input embedding of \eqref{eq:inference} by:
\begin{equation}
\frac{1}{K}\sum_{i=1}^K f_\theta(x \oplus x_i^{\text{aug}}),
\end{equation}
where $\oplus$ is the concatenation operator for text and $x_i^{\text{aug}}$ is the $i$-th sample from $P_g(\cdot | x)$. The mean of the embeddings summarize the information induced by LLM. 

\paragraph{Constructing positive training pairs.} We propose a contrastive loss between input text and generated text as another training objective. Let $\mathit{I}$ be a training batch and $A(i)$ 
be the set of augmented texts with the same pseudo-label as input $x_i$. Our objective encourages proximity between $x$ and $x^{\text{aug}}$ (sampled from $A(i)$) in the embedding space:
\begin{equation} 
\fontsize{10pt}{12pt}
\begin{split}
     &\mathcal{L}_{t2g}=\sum_{i \in \mathit{I}} \frac{ -1}{|A(i)|} \\ 
     &\sum_{x^{\text{aug}} \in A(i)}\log{\frac{\exp({sim(f_{\theta}(x_i),f_{\theta}(x^{\text{aug}}) )})}{\sum_{j \in \mathit{I}}\exp(sim(f_{\theta}(x_i),f_{\theta}(x_j) ))}}.
\end{split} \label{eq:t2g}
\end{equation}


\begin{algorithm*}[ht!]
  \caption{Self-training with GPT assisted in the loop}
  \textbf{Require:} Unlabeled texts $X$, label descriptions $C$, instruction-tuned GPT model $g(\cdot)$.\\
  \textbf{Initialization:} Classifier $f_{\theta}(\cdot)$ initialized with pre-trained sentence encoder. Empty dictionary $\operatorname{GenDict}$ to cache conditional generated text.  \\
  \textbf{Input augmentation}: For each observed text, generate $K$ samples of augmented text from $P_g(\cdot | x)$.\\
  \For {$t: 1 \rightarrow T$ self-training iterations}
  {
    Use $f_{\theta}(\cdot)$ to generate pseudo-labels $\hat{y}$ (eq.\ref{eq:inference}) and soft-target $Q$ (eq.\ref{eq:soft-label}) for texts with input augmentation in Section.\ref{subsec:input_augment}. Sample a balanced subset of pseudo-labeled training pairs of size $S_t$ according to prediction confidence\;
    \For {each training sample $(x, \hat{y})$}
    {
        \eIf{key $(x, \hat{y}) \in \operatorname{GenDict}$}
        {
            Fetch generated texts from $\operatorname{GenDict}$ \Comment{Use cached generated text}\;
        }
        {
            Generate $M$ samples from $P_g(\cdot | x, \hat{y})$\Comment{Conditional augmentation in Section~\ref{subsec:condition_aug}}\;
            Add generated texts to $\operatorname{GenDict}$\Comment{Cached generated text}\;
        }
    }
    Use sampled training pairs and the conditionally generated text to update the parameters $\theta$ of $f_{\theta}(\cdot)$ with the objective function $\mathcal{L} = \mathcal{L}_{g2l} + \mathcal{L}_{t2g}$ from \eqref{eq:t2g} and \ref{eq:g2l}.
  }
\label{alg:self-train}
\end{algorithm*}

\subsection{Crafting Training Pairs with LLM}
\label{subsec:condition_aug}
Self-training can introduce bias into a classifier due to mislabeling instances. To address this issue, we propose to generate high quality pseudo-labeled data pairs, as shown in \figref{fig:conditional_aug_illustration}.
Consider an instance where an article about the retirement of Starbucks' president, whose true label is "business", is mistakenly labeled as "sports". Training the model with this incorrect label blurs the distinction between the business and sports categories.

To mitigate this issue, we employ the LLM to conditionally augment the input text based on the sports category. This is achieved by framing instructions like, "Discuss the sports aspects of the article". Consequently, the produced text mirrors typical articles within the sports category. By optimizing this newly generated text, instead of the original mislabeled instance, we correct its placement relative to the decision boundary separating "sports" and "business". Essentially, by creating texts based on pseudo labels, we synthesize training pairs that enhance the separation of class labels in the embedding space, thereby addressing the challenges of mislabeling inherent to self-training.

Let $x^\text{cond}$ be the conditionally augmented text, the modified \eqref{eq:loss} is:
\begin{equation}
\fontsize{10pt}{12pt}
    \mathcal{L}_{g2l} = -\sum_{i=1}^{N} \sum_{j=1}^{L} Q(\hat{y}_j|x_i^{\text{cond}})\log P(\hat{y}_j|x_i^{\text{cond}})
    \label{eq:g2l}
\end{equation}

\subsection{Algorithm for Self-training}
We apply self-training with \eqref{eq:t2g} and \ref{eq:g2l} in an iterative way as shown in \Algref{alg:self-train} with LLM assisting in the loop. During training, we found that a balanced sampling that keeps the same number ($S_t$ for iteration $t$) of training for each category is important for the stability of self-training. Additionally, we use a dictionary $\operatorname{GenDict}$ to cache the conditional generated text to avoid repeated generation for better efficiency.

\section{Experiments}
\label{sec:experiments}
\subsection{Datasets and Experimental Settings}
\begin{table*}[ht!]
\begin{adjustbox}{width=\linewidth,center}
\begin{tabular}{ccccccc}
\toprule
Dataset & Classification Type & \#Classes & \#Orig Train & \#Our Train & \#Test & Avg Length \\
\hline
AG News & News Topic & 4 & 120,000 & 4,000 & 7,600 & 38 \\
DBPedia & Wikipedia Topic & 14 & 560,000 & 11,200 & 70,000 & 50 \\
Yahoo Answers & Question Answering & 10 & 1,400,000 & 15,000 & 60,000 & 70 \\
Amazon & Product Review Sentiment & 2 & 3,600,000 & 20,000 & 400,000 & 78 \\
\bottomrule
\end{tabular}
\end{adjustbox}
\caption{Statistics of datasets for multi-class text classification.\label{tab:data_stats}}
\end{table*}

We conduct experiments on $4$ benchmark text classification datasets: AG News, DBpedia, Yahoo Answers and Amazon, with the statistics shown in \tblref{tab:data_stats}. In the experiments, we initialize our sentence encoder with supervised SimCSE Roberta-base model (110M parameters)~\cite{gao2021simcse}. For the generative model, we use the Alpaca-7B~\cite{alpaca} as our choice of LLM, which is a GPT model fine-tuned with human instructions~\cite{touvron2023llama}. The label prompts and the instruction template are illustrated in \tblref{tab:prompt_design} in Appendix. Please refer to \secref{apd:experiment} in Appendix for implementation details.



\label{subsec:main results}
\begin{table*}[ht!]
\begin{tabular}{ccccccc}
\toprule
ID & Self-train & Methods           & AG News       & DBpedia & Yahoo Answers & Amazon \\ \Xhline{2\arrayrulewidth}
1 & --         & Supervised        & 94.2          & 99.3    & 77.3          & 97.1   \\ \hline
2 & No         & SimCSE (Sentence-enc)  & 74.5          & 73.8    & 55.6          & 88.8   \\
3 & No         & Alpaca-7B (LLM)  & 77.4        & 60.6    &   52.1       & 86.6   \\
4 & Yes        & iPET              & 86.0          & 85.2    & 68.2          & 95.2   \\
5 & Yes        & LOTClass          & 86.4          & 91.1    & --           & 91.6   \\ 
\Xhline{2\arrayrulewidth}
\hline
6 & --        & Supervised-downsample* & 93.8 & 98.7    &  76.5 & 97.0\\
\hline
7 & Yes & PESCO* & 85.0 &	96.6 & 65.8	& 92.4 \\
8 & Yes  & \ourmodel* & \textbf{89.2} & \textbf{98.3}  & \textbf{68.7}          &   \textbf{95.4}     \\
\hline
9 & Yes & \ourmodel*  - CA & 87.5 & 97.6 & 65.1 & 94.3  \\
10 & Yes & \ourmodel*  - IA & 86.2 & 97.1 & 63.5 & 93.6\\
11 & Yes & SimCSE + Self-training (Eq \ref{eq:loss}) & 83.2	& 94.3	& 62.7	& 91.5 \\
\bottomrule
\end{tabular}
\caption{Comparison of classification methods on benchmark datasets. The test accuracy of best performing zero-shot method is highlighted in bold phase. Row 7-11 (with *) use a down-sampled dataset with 4k (3.4\%), 11.2k (2\%), 15k ($<$1\%), 20k ($<$1\%) unlabeled training instances respectively. Rows 9-11 are ablation tests with input augmentation (IA) or conditional augmentation (CA) removed. \label{tab:main_results}}
\end{table*}

\subsection{Baseline Methods}
\noindent \textbf{Alpaca-7B} is a LLM baseline for zero-shot classification. We solicit the LLM for zero-shot classification with the instruction "Classify the text by outputting a single category from [label categories]".

\noindent \textbf{iPET}~\cite{schick2020s} formulates zero-shot text classification as a cloze test, where a pre-trained BERT~\cite{devlin2018bert} model is used to predict the output label(s) by completing a prompt such as ``This article is about \_", which is concatenated right after an input document. An iterative self-training algorithm is used in iPET to improve the model for better generalization.

\noindent \textbf{LOTClass}~\cite{meng2020text} applies the BERT model to extract keywords related to the label names from unlabeled texts and then create pseudo labels based on the extracted keywords. LOTClass also applies a self-training algorithm to further improve the classification performance.

\noindent \textbf{PESCO}~\cite{wang2023pesco} formulates zero-shot classification as sentence alignment and uses contrastive self-training to improve the model performance. As an augmentation, it selects salient sentences from documents to create additional positive training pairs. 

\subsection{Experimental Results}
In \tblref{tab:main_results}, we present a comparison of the test accuracy of our model with other baselines on four benchmark classification datasets. Specifically, rows 1-5 are experiments using the entire (unlabeled) training set and rows 6-11 use a down-sampled dataset with 4k (3.4\%), 11.2k
(2\%), 15k (<1\%), 20k (<1\%) unlabeled training instances from the original datasets respectively.

\noindent \textbf{Comparison with Alpaca-7B}: While Alpaca-7B (row 3) has demonstrated strong instruction following ability to solve problems without any training, it exhibits lower performance compared to \ourmodel (row 8) and other self-training methods on classification task. The reason could be attributed to the domain adaptation effect of self-training. Classification tasks involve comparing instances, such as an article being more likely to belong to the ``sports" category when compared to articles in the ``business" category. In our analysis in \secref{subsec:self-training}, self-training enforces the separation between classes to improve the generalization ability. This can be further supported when the number of classes increases in DBpedia and Yahoo Answers dataset, the performance of Alpaca gets worse. Furthermore, Alpaca-7B takes 9 minutes per 10k instances on one A6000 gpu while \ourmodel takes 10 seconds, which is roughly x50 speed up.



\begin{figure*}[ht!]
    \centering
    \includegraphics[width=0.4\linewidth]{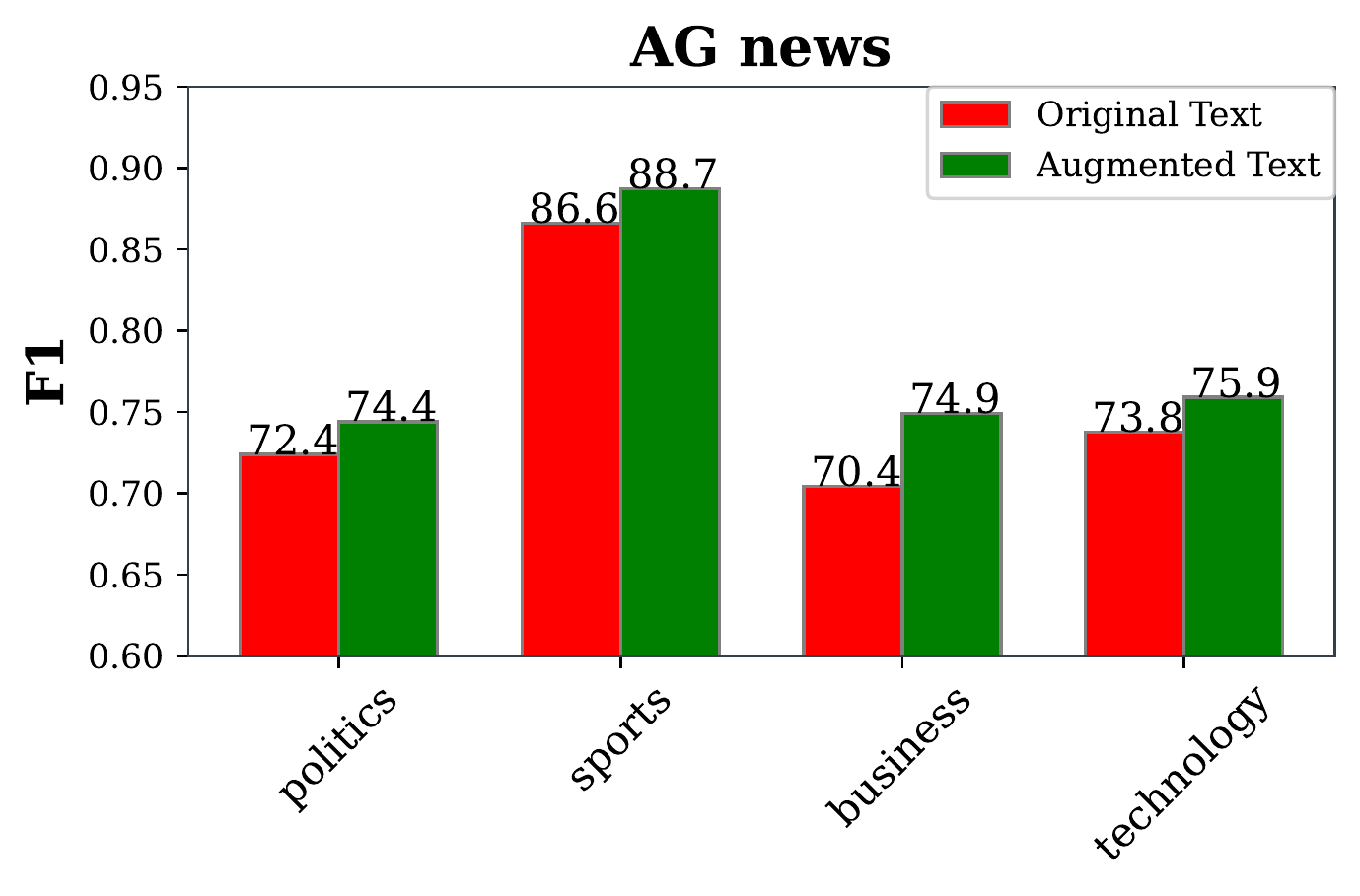}
    \includegraphics[width=0.4\linewidth]{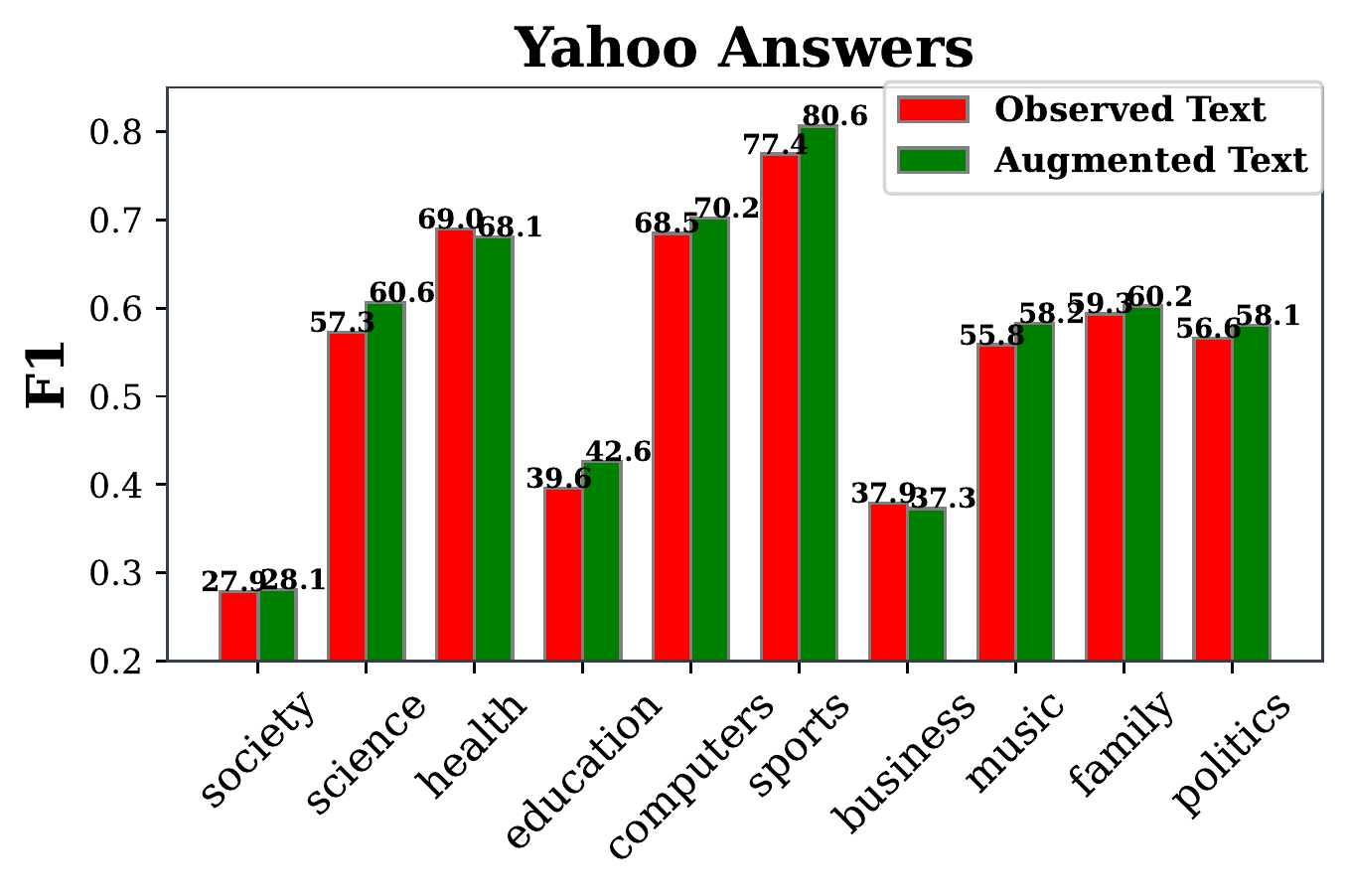}
    \includegraphics[width=0.4\linewidth]{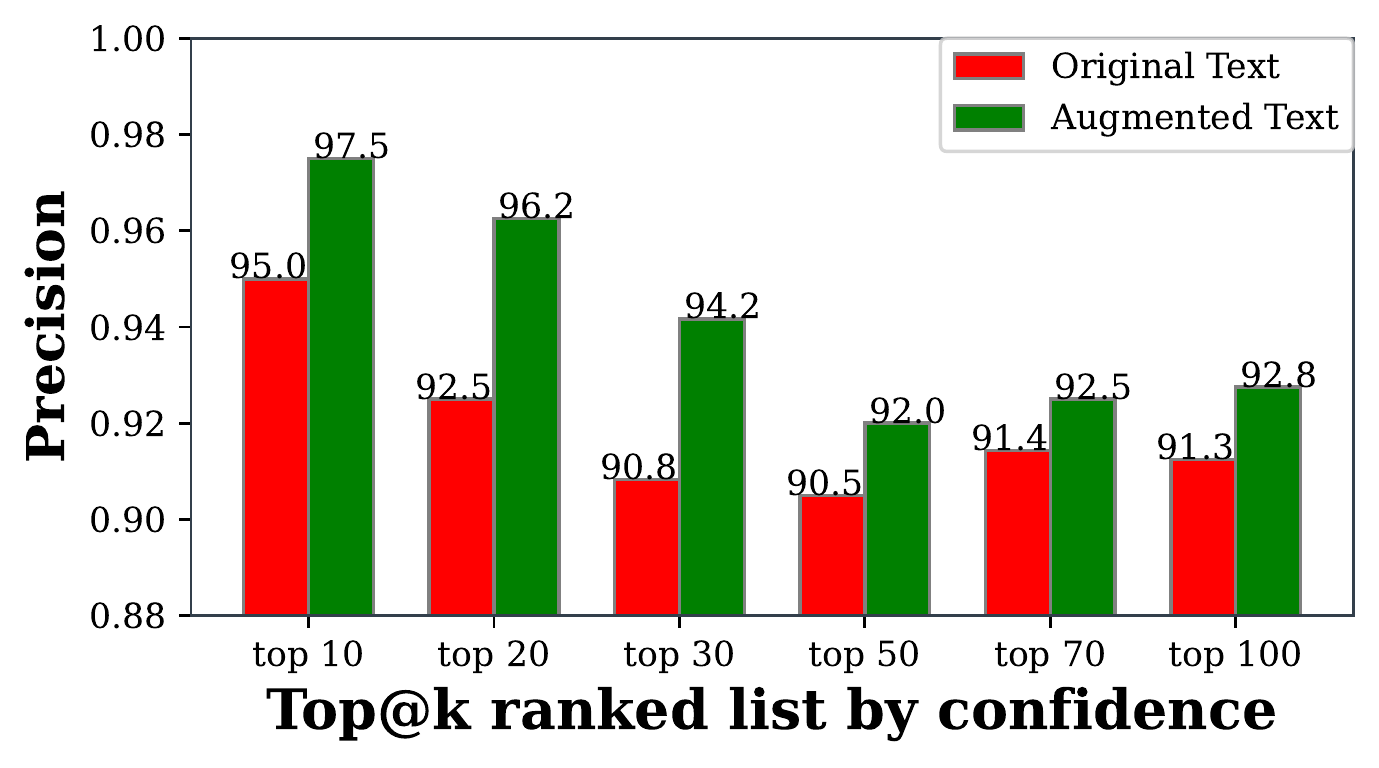}
    \includegraphics[width=0.4\linewidth]{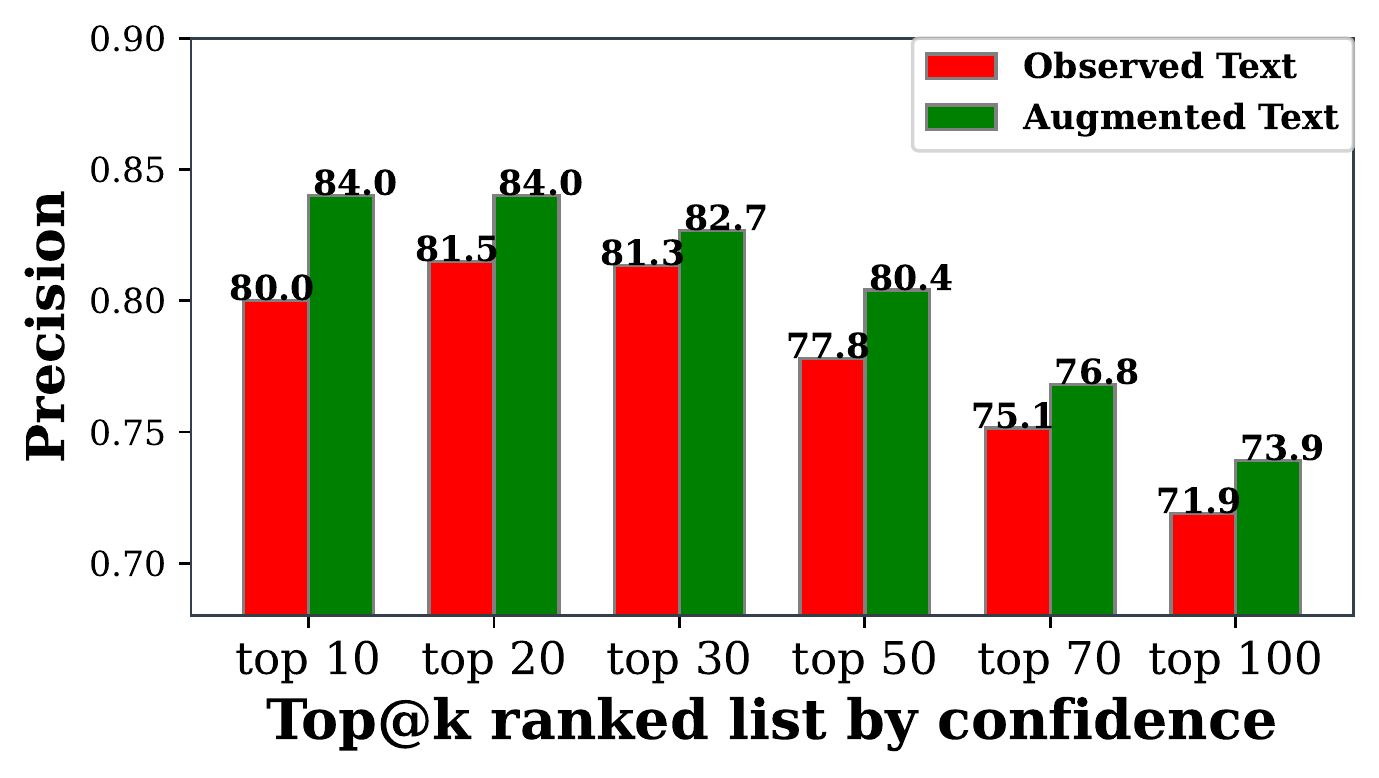}
    \vspace{-0.5cm}
    \caption{Per class F1 (upper) and ranking-based precision (lower) for classification performance with input augmentation.}
    \label{fig:exp_general_aug}
\end{figure*}

\noindent \textbf{Comparison with SOTA Methods}: 
Both iPET (row 4) and LOTClass (row 5) use self-training algorithm for zero-shot classification, but \ourmodel outperforms the previous self-training methods even with significantly fewer instances ($< 5\%$ of original size). 
The iPET model improves pseudo label prediction with an ensembling about 15 models to reduce prediction variance. In comparison, our approach improves pseudo label prediction by ensembling augmented text embedding during self-training, leading to improved performance and a more memory efficient alternative.
While LOTClass uses a BERT model to extract keywords for each category as an augmentation, it is less expressive than using an LLM to generate coherent human language as augmentation. 
PESCO (row 7) is the most recent SOTA with contrastive self-training and introduced an augmentation technique by learning on salient sentences. However, the method still requires a large amount of data to be effective. In scenarios where only a limited number of unlabeled texts are available, PESCO still underperforms our model. 

\noindent \textbf{Effectiveness of Contrastive Self-training}: Row 2 represents the sentence encoder baseline with SimCSE, whereas row 11 represents SimCSE + contrastive self-training algorithm as per \eqref{eq:loss}. The result shows that incorporating contrastive self-training leads to significant gains. Compare row 3 (Alpaca-7B) with row 11. Despite being a larger model in scale, Alpaca-7B still outperforms the self-training approach across all benchmark datasets, underscoring the effectiveness of class separation with self-training for classification task.


\subsection{Analysis of LLM Augmentation}
In this section, we denote the input augmentation in \secref{subsec:input_augment} as IA and the conditional augmentation based on pseudo label in \secref{subsec:condition_aug} as CA.
Rows 9 and 10 in \tblref{tab:main_results} shows ablation tests with CA and IA removed. Overall, our LLM data augmentation, with and without conditioning on pseudo label, both lead to improved performance, due to their ability to provide more accuracy pseudo label and high quality synthetic training pairs.

\noindent \textbf{Effectiveness of IA}:
In this evaluation, we investigate the effectiveness of input augmentation for first round pseudo-labeling \textit{without training}. We evaluate the performance of our model on two datasets, namely AG News and Yahoo Answers, using two evaluation metrics: per class F1 metric and ranking-based precision metric according to prediction confidence. The per class F1 metric provides an insight into how well the model performs on each individual class by balancing precision and recall. In the upper part of \figref{fig:exp_general_aug}, our findings indicate that LLM augmented data leads to improved performance across all categories for AG News and in eight out of ten classes for Yahoo Answers.

In the lower part of \figref{fig:exp_general_aug}, we employ a ranking-based precision metric to assess the quality of the most confident cases. Our results demonstrate that using augmented data yields better precision for the most confident cases. Notably, our study on the Yahoo Answers dataset indicates that the predictions are better calibrated with the use of augmented data, implying that highly confident samples exhibit better precision. Conversely, such a trend was not observed in unaugmented data, where the top 30 had higher accuracy than the top 10. Better calibration justifies the sampling from the most confident pools for self-training, making it a more reliable method for improving model performance.

\begin{table*}[th!]
\begin{adjustbox}{width=0.95\linewidth,center}
\begin{tabularx}{\textwidth}{c|X}
\hline
\makecell{Original Input Text}  &  Starbucks Corp's president and chief executive, Orin Smith, said Tuesday he plans to retire early next year because he wants to slow down and focus on philanthropy, family and sports. \\
\hline
\makecell{Conditioned on Politics} & The announcement of Orin Smith's retirement is likely to have \high{political implications} for Starbucks Corporation... \high{His replacement} will have big shoes to fill, and the decision could have long-term effects on the \high{company's policies and direction}.
 \\
\hline
\makecell{Conditioned on Sports} & ...\high{Sports} have always been a major part of Smith's life, as he was a \high{college athlete} and later went on to become the CEO of Starbucks. It is clear that sports have had a major influence on his life and he wants to make time for them in his retirement. \\
\hline
\makecell{Conditioned on Business} & Starbucks Corp has seen great success under the leadership of Orin Smith, with the company's \high{stock price more} than tripling since he became CEO in 2005. This success has allowed him to retire early and …
 \\
\hline
\makecell{Conditioned on Technology} & Orin Smith's plan to retire early next year is an example of how \high{technology} has changed the way we work and live. By utilizing technology, Smith is able to take advantage of the increasingly popular trend of ``work-life balance" ... 
 \\
\hline
\end{tabularx}
\end{adjustbox}
\caption{Examples of generated text conditioned on pseudo labels in the left column.}
\label{tab:example_cond_gen}
\end{table*}

\noindent \textbf{Effectiveness of CA}:
\label{sec:analysis_cond_gen}
To study the quality of conditional generation based on class labels, we first present examples of generated texts from an sample in AG News dataset, shown in \tblref{tab:example_cond_gen}. Each example is a cherry-picked sample out of five random samples. New training instances are crafted by generating augmented text conditioned on each label. The augmented text expands on specific aspects of the label while retaining the meaning of the original input text.

\begin{figure*}[th!]
    \centering
    \includegraphics[width=0.4\linewidth]{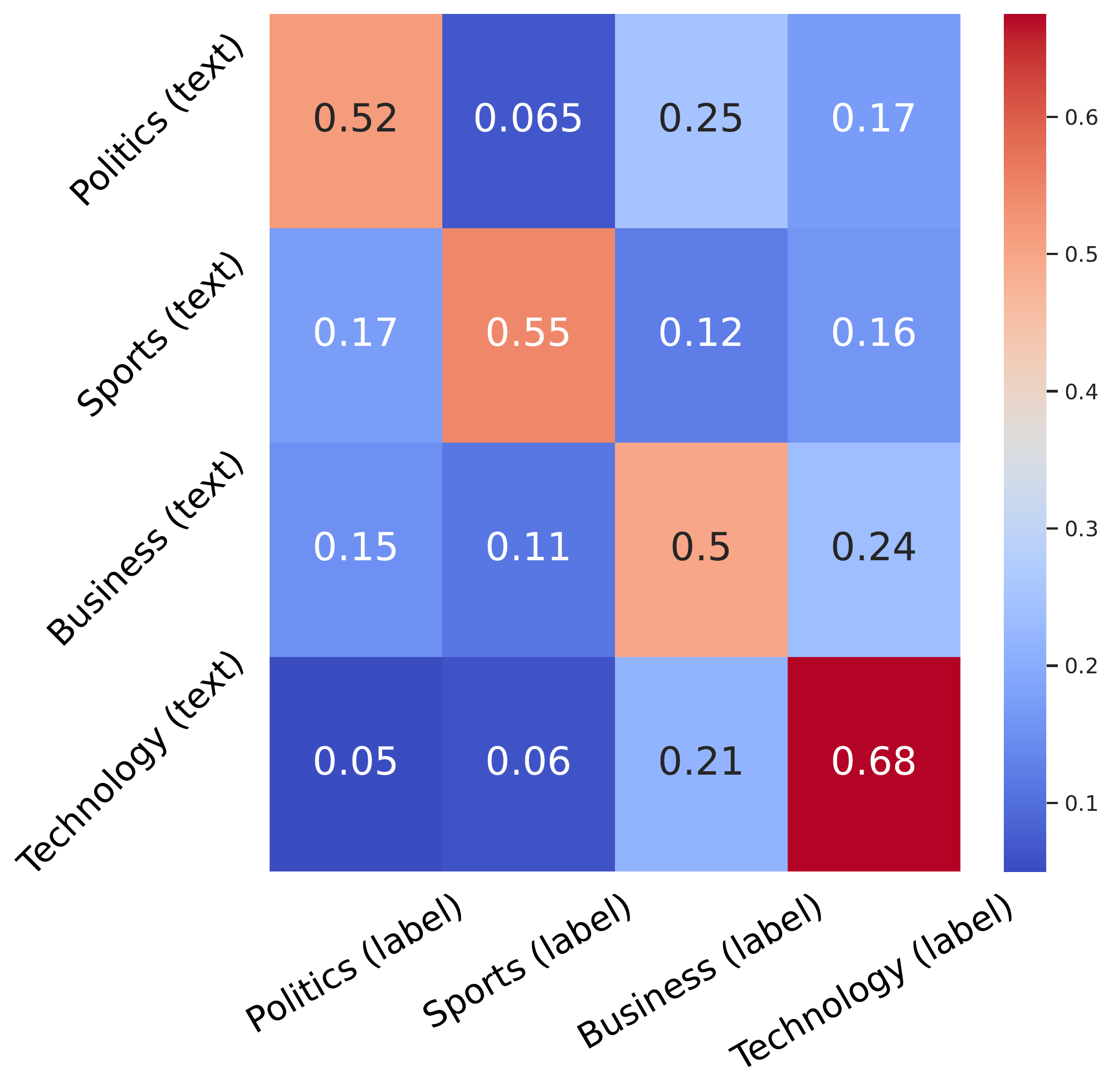}
    \includegraphics[width=0.5\linewidth]{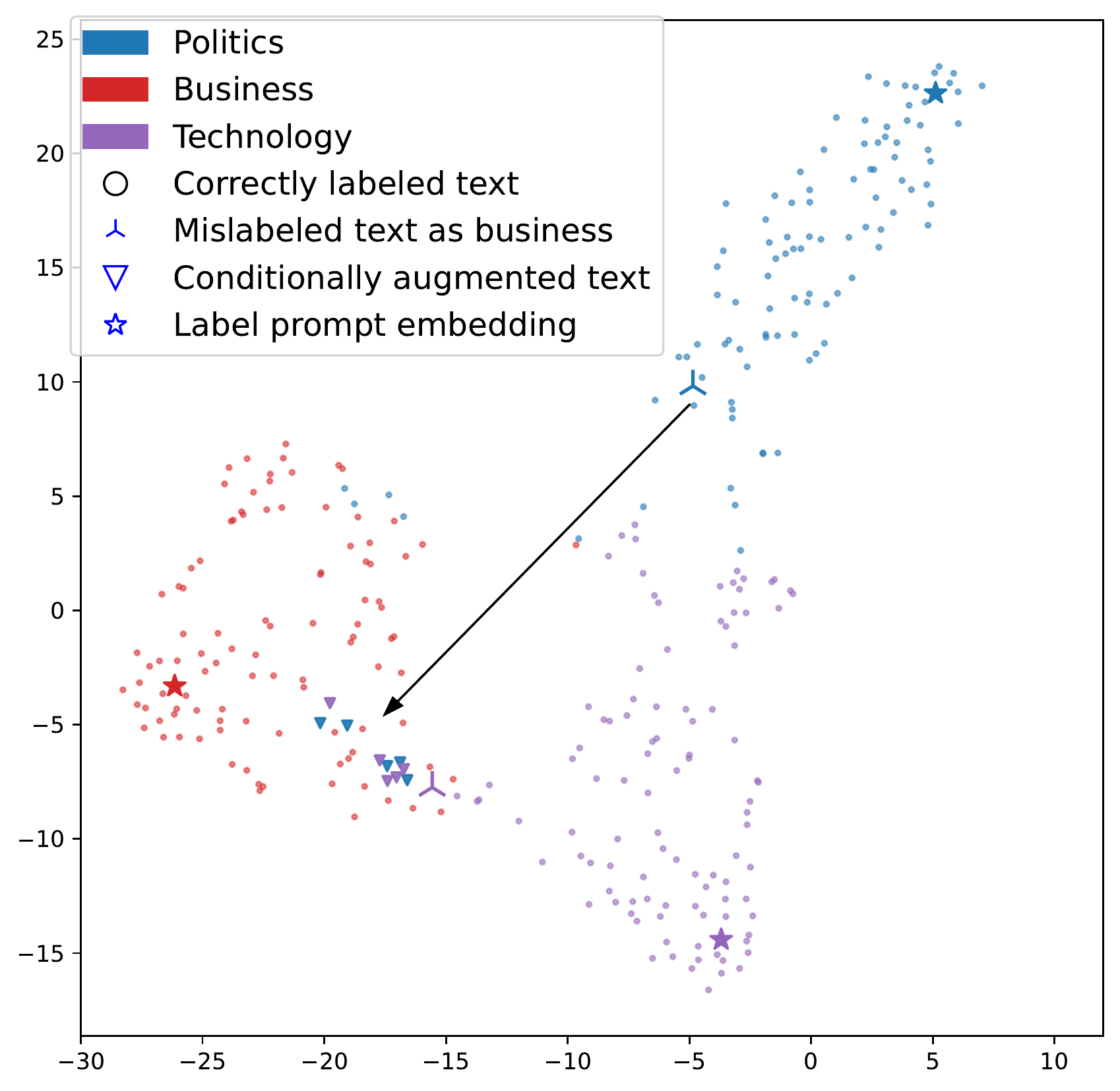}
    \caption{The left figure shows a heatmap of the probability when a conditionally generated text based on pseudo label aligns with each of the label prompts. The right figure shows the distribution of the generated text plotted using T-SNE (sports category is out of scope).}
    \label{fig:ablation_cond_aug}
\end{figure*}

In the left of \figref{fig:ablation_cond_aug}, we show a heatmap of the probability when a conditionally generated text (vertical) aligns with the corresponding label class (horizontal). The highest probability occurs along the diagonal, indicating that the conditionally augmented text based on pseudo label has a closer meaning to the corresponding label class. In the right of \figref{fig:ablation_cond_aug}, we plot the distribution of the generated text plotted using T-SNE. The embeddings were obtained by our sentence encoder trained on the $100$-th (out of $1000$) iteration. We selected two instances that were misclassified as business and located close to the decision boundary. The augmented text, conditioned on the business category, was found to be closer to the label prompt embedding of the business category. This demonstrates the effectiveness of our method to generate less confusing training pairs away from the decision boundary and closer to the pseudo label centroid.

\section{Related Work}

\noindent \textbf{Knowledge Distillation from GPT: }
To leverage the language modeling power of large model, previous work distills knowledge from LLM to improve downstream classification performance~\cite{shridhar2023distilling, sahu2022data, wei2021few, sun2020contrastive, honovich2022unnatural,chiang2023vicuna}, or directly generates text and label pairs~\cite{yoo2021gpt3mix,ye2022zerogen,meng2022generating} to train a classifier for downstream tasks. However, generating training data from scratch can lead to low-quality data with unrelated or ambiguous examples analyzed in~\citep{gao2022zerogen}. Our generation is grounded in the context of the corpus with enrichment in semantic and diversity, providing a practical alternative to generation-based methods for zero-shot text classification and knowledge distillation.

\noindent \textbf{Zero-shot Text Classification: }
Zero-shot text classification involves classifying text into categories without the aid of pre-labeled examples. This domain has seen various approaches, including clustering-based methods that utilize text embeddings to generate robust pseudo-labels ~\cite{cho2023celda, fei2022beyond}. Other models adopt annotation generation by identifying keywords akin to label descriptions~\cite{meng2020text} or employing keyword correlation graphs~\cite{zhang2021weaklysupervised}. A notable recent work CeLDA ~\cite{cho2023celda}, which achieves comparable results to our model. However, its effectiveness is relied on using a large model (T5 11B) and large amount in-domain text available~\cite{cho2023celda}.

In contrast, our work formulates zero-shot text classification as sentence alignment, a perspective shared by several recent studies ~\cite{gao2021simcse, hong2022tess, shi2022nearest, wang2023pesco, zhang-etal-2023-long}. This approach typically employs contrastive learning for training sentence encoders, aiming to optimize representations by minimizing distances between semantically similar inputs and maximizing distances between dissimilar ones in the embedding space. Our model innovatively uses a large language model (LLM) to generate training pairs, facilitating the training of a robust classifier for zero-shot text classification, even with a limited number of instances available.

\subsection{Self-training Methods}
Self-training methods~\cite{van2020survey} have been proposed as a semi-supervised approach to improve a classifier from unlabeled datasets, where predictions on unlabeled data are used to fine-tune the classifier~\cite{lee2013pseudo}. To improve the pseudo label quality, previous work~\cite{gera2022zero} use a small set of instances with the most confident prediction for self-training. LOTClass~\cite{meng2020text} improves the quality of pseudo label by an expansion of label vocabulary using BERT and iPET~\cite{schick2020s} ensembles multiple version of model at different stage of training. Our work improves self-training by generating augmented text with instruction-tuned LLM in the training loop.

\section{Conclusion}
In conclusion, our proposed approach, GenCo, effectively addresses the difficulties and limitations of using LLMs directly for zero-shot text classification. By leveraging the generative power of an LLM in a self-training loop of a smaller, sentence encoder classifier with contrastive learning, \ourmodel outperform state-of-the-art methods on four benchmark datasets. Our approach is particularly effective when limited in-domain text data are available. The success of our approach highlights the potential benefits of incorporating the generative power of LLM into iterative self-training processes for smaller zero-shot classifiers. We hope that our work will inspire further research in this direction, ultimately leading to more efficient and effective NLP models.

\section{Limitations}
The main goal of our paper is to promote the usage of LLMs (Alpaca-7B in our case) to assist in training of a smaller model (Roberta-SimCSE) on zero-shot classification tasks. The proposed loss function and corresponding analysis is shown in \tblref{tab:select_temperature} in Appendix, but we mainly use that as a theoretical motivation of leveraging decision boundaries between classes.

Another part is data efficiency. We have shown that using GPT generated data can alleviate the data hungry issue for deep learning models. However, when there is abundant of data, generating training instances with LLM can be expensive with less gains. Also, due to compute and buget limitations, we didn't use larger LLMs for our experiments, as an estimiated cost will be around 150\$ per dataset with the GPT-3.5 at time of writing.

Finally, we realize that more tricks and engineering designs are employed in our experiments and please refer to our code for reference.

%
%
%
\bibliography{custom}
\bibliographystyle{acl_natbib}

\appendix
\section{Experiments}
\label{apd:experiment}
\subsection{Implementation Details}
The label prompts are shown in the upper part of \tblref{tab:prompt_design}. The label prompts are similar to the ones used in in PESCO~\cite{wang2023pesco}. We solicit LLM for text augmentation with the instruction template in the lower part of \tblref{tab:prompt_design}, which is the same ones used for Alpaca fine-tuning.

For the generation parameters, we used $temperature$=0.8, $top\_p$=0.95, and sample $K$=5 augmented texts for each instance with $min\_length=64$ and $max\_length=128$. 
For the self-training of sentence encoder model, we used $batch\_size$=$3*|C|$ ($|C|$ is the number of categories), $lr$=1e-5, the max length is $128$ for AG News and DBPedia and $192$ for Yahoo Answers and Amazon.
All the experiments are performed on NVIDIA RTX A6000 gpus. Please refer to our code for details.

\begin{table}[ht!]
\small
\begin{tabular}{l}
\toprule
\textbf{Label Prompt} \\
\makecell[l]{(1)Category: [label]. \\ (2)It is about [label].} \\
\hline
\textbf{Instruction-based (Conditional) Augmentation} \\
\makecell[l]{Below is an instruction that describes a task,  paired \\ with an input that provides further context. Write a \\ response that appropriately completes the request.\\\#\#\# Instruction:\\ Elaborate the text in a few sentences. \\ (Discuss the [pseudo label] aspects of the article.) \\ \#\#\# Input:\\ $[$text$]$ \\ \#\#\# Response:
}\\
\bottomrule
\end{tabular}
\caption{The designed prompts for enhanced label description and conditional augmentation based on pseudo label.\label{tab:prompt_design}}
\end{table}

\subsection{Selection of Temperature in Eq \ref{eq:loss}}
\label{apd:temperature}
As shown in \tblref{tab:select_temperature}, we include the results with over 5 runs on each dataset. We found $\tau = 0.1$ to be a reasonble choice with slightly better performance, but we acknowledge that the difference is rather small, sometimes fall within std. The choice of $\tau$ may serve more of a theoretical motivation rather than practically concerns (as acknowledged in limitation). The theoretical framework unifies previous soft labeling approaches in~\cite{meng2020text,wang2023pesco} and is easier for the proof of theorem.

\begin{table*}[ht!]
\begin{tabular}{c|cccc}
\toprule 
       & Agnews        & DBpedia       & Yahoo Answers & Amazon        \\
\hline
$\tau$=1.0  & 82.75 $\pm$ 0.06 & 93.77 $\pm$ 0.07 & 62.66 $\pm$ 0.06 & 91.39 $\pm$ 0.06 \\
$\tau$=0.5  & 83.04 $\pm$ 0.05 & 94.19 $\pm$ 0.05 & 62.70 $\pm$ 0.10 & 91.44 $\pm$ 0.06 \\
$\tau$=0.1  & \bf 83.18 $\pm$ 0.05 & 94.29 $\pm$ 0.05 & 62.74 $\pm$ 0.08 & \bf 91.48 $\pm$ 0.05 \\
$\tau$=0.05 & 83.03 $\pm$ 0.05 & \bf 94.34 $\pm$ \bf 0.03 & \bf 62.77 $\pm$ 0.10 & 91.42 $\pm$ 0.04 \\
$\tau$=0.01 & 83.02 $\pm$ 0.05 & 94.33 $\pm$ 0.03 & 62.76 $\pm$ 0.11 & 91.42 $\pm$ 0.04  \\
\bottomrule 
\end{tabular}
\caption{For the choice of temperature $\tau$ in \eqref{eq:loss}, we include the results with over 5 runs on each dataset. We found $\tau = 0.1$ to be a reasonble choice with slightly better performance, but we acknowledge that the difference is rather small, sometimes fall within std. \label{tab:select_temperature}}
\end{table*}

\subsection{Inference Time Augmentation}
While \ourmodel doesn't require LLMs during inference, in our ablation test in \tblref{tab:ablation}, we study the impact of inference time augmentation (assuming GPT is available at test time) and self-training on the performance metric. To test inference time augmentation, we performed experiments on a downsampling of both training and testing instances.

Our results show that inference time augmentation (rows with "IA") leads to a performance gain of $1$-$2\%$, with a more substantial improvement observed for AG News and Yahoo Answers. This may be attributed to the fact that AG News has an average text length of only $38$ words, and the Yahoo Answers dataset includes many answers with only one phrase. Inference time augmentation effectively enhances the quality of shorter text inputs.

\begin{table*}[ht!]
\begin{tabular}{ccccccc}
\toprule
ID & Self-train & Methods           & AG News       & DBpedia & Yahoo Answers & Amazon \\ 
\Xhline{2\arrayrulewidth}
\multicolumn{3}{c}{\# unlabeled train} & 4k (3.4\%) & 11.2k (2\%)& 15k ($<$ 1\%) & 20k ($<$ 1\%) \\
\multicolumn{3}{c}{\# test} & 7.6k & 28k & 20k & 20k  \\
\hline
1 & No & Sentence-enc & 75.6  & 73.4  & 55.5 &  89.6 \\
2 & No & Sentence-enc $+$ Inf-Aug  & 78.2 & 74.7 & 57.4 & 90.2 \\
\hline
3 & Yes & Self-train & 83.3 & 96.3 & 62.5 & 91.1 \\
4 & Yes & Self-train $+$ Inf-Aug & 83.9 & 96.8 & 64.3 & 91.3 \\
5 & Yes & \ourmodel   & 89.2 & 98.4 &  68.6 &  95.3  \\ 
6 & Yes & \ourmodel $+$ Inf-Aug & 89.7 & 98.5 &  70.2 &  95.4  \\ 
\bottomrule
\end{tabular}
\caption{Evaluation of inference time augmentation. "Inf-Aug" represents input augmentation added during inference.\label{tab:ablation}}
\end{table*}

\subsection{Qualitative Examples for Conditionally Generated Examples on Pseudo-label}

In \tblref{tab:example_cond_gen}, we show generated examples of a sample text from the Agnews dataset. We generate $5$ examples conditioned on each of the $4$ labels, and cherry-pick one for each label in the table presentation. The example shows that the topic of a generated text is related to the label which is conditioned on, while pertains the original meaning. This opens a path to leverage the language understanding ability of LLM for data augmentation, especially during self-training.

\onecolumn
\section{Proof of Theorems}
\label{apd:proofs}
\begin{theorem}
\label{theorem:1}
Consider a binary classification problem with linearly separable labeled examples, when $0<\tau<1$, optimizing $\mathcal{L}_{t2l} = -\sum_{i=1}^{N} \sum_{j=1}^{L} Q(\hat{y}_j|x_i)\log P(\hat{y}_j|x_i)$ with gradient descend will enforce the larger margin between classes. 
\end{theorem}

\begin{proof}
We use dot product $\langle \cdot, \cdot \rangle$ as implementation of similarity function. 
Let the embedding of instance $i$ be $\vx_i = f_\theta(x_i)$ and the embedding of label prompt $j$ be $\ve_c = f_\theta(p_c), c\in \{1,2\}$ for binary classification. Then,
\begin{align}
    P(\hat{y}_1 | x_i ;\theta) &=  \frac{\exp(\langle\vx_i, \ve_1\rangle)}{
    \exp(\langle\vx_i, \ve_1\rangle) + 
    \exp(\langle\vx_i, \ve_2\rangle)} = \frac{1}{1 + \exp(-\langle\vx_i, \ve_1 - \ve_2 \rangle)} \label{eq:1} \\
    P(\hat{y}_2 | x_i ;\theta) &= 1 - P(\hat{y}_1 | x_i ;\theta)
\end{align}
Notation-wise, define $d_i=\langle\vx_i, \ve_1 - \ve_2\rangle$, then 
\begin{align}
    P(\hat{y}_1 | x_i ;\theta) &=  \frac{1}{1 + e^{-d_i}}  \\
    P(\hat{y}_2 | x_i ;\theta) &= 1 - \frac{1}{1 + e^{-d_i}}  \\
\end{align}
In binary classification, the margin is simply
$$
    \text{margin} =
    \begin{cases}
    d_i & x_i \text{ is class 1} \\
    -d_i & x_i \text{ is class 2}
    \end{cases}
$$
For soft-label distribution $Q$, 
\begin{align}
    Q(\hat{y}_1 | x_i ;\theta) &=  \frac{1}{1 + e^{-d_i/\tau}}  \\
    Q(\hat{y}_2 | x_i ;\theta) &= 1 - \frac{1}{1 + e^{-d_i/\tau}}  \\
\end{align}
Then $\mathcal{L}_{t2l}$ is derived as
\begin{equation}
    \mathcal{L}_{t2l} = \sum_{i=1}^{N} \log(1+e^{-d_i}) + \frac{d_i e^{-d_i/\tau}}{1+e^{-d_i/\tau}}
\end{equation}
Calculate the derivative of $\mathcal{L}_{t2l}$ w.r.t $d_i$,
\begin{equation}
    \frac{\partial \mathcal{L}_{t2l} }{\partial d_i} = \frac{-d_ie^{-d_i/\tau}}{\tau(e^{-d_i/\tau} + 1)^2} + \frac{e^{-d_i/\tau} - e^{-d_i}}{(e^{-d_i/\tau} + 1)(e^{-d_i}+1)} \label{eq:derivative}
\end{equation}
For the first part of \eqref{eq:derivative}, the sign depends on $-d_i$. For the second part, the sign depends on $e^{-d_i/\tau} - e^{-d_i}$. When $0 < \tau < 1$,
$$
    \begin{cases}
        e^{-d_i/\tau} - e^{-d_i}  < 0 & \text{when } d_i > 0  \\
        e^{-d_i/\tau} - e^{-d_i} > 0 & \text{when } d_i < 0
    \end{cases}
$$
Therefore, 
\begin{equation}
\begin{cases}
\frac{\partial \mathcal{L}_{t2l} }{\partial d_i} < 0 & \text{when } d_i > 0  \\
\frac{\partial \mathcal{L}_{t2l} }{\partial d_i} > 0 & \text{when } d_i < 0
\end{cases}\label{eq:gd}
\end{equation}
One step of gradient descend optimizes $d$ by 
$d_i^\prime = d_i - \eta \frac{\partial \mathcal{L}_{t2l} }{\partial d_i}$. From \eqref{eq:gd}, we get the conclusion that $|d_i^\prime| > |d_i|$. In other words, the margin becomes larger after optimization, which finishes the proof.
\end{proof}

\begin{theorem}
Under the setting in Theorem \ref{theorem:1}, let $m_i$ be the margin of instance $i$ and consider the constraint $m_i \le B$ for all $i$, the classifier converges to a max margin classifier, as the bound $B$ goes to infinity.
\end{theorem}

\begin{proof} 
Using the definition from Theorem \ref{theorem:1},
\begin{equation}
    \mathcal{L}_{t2l} = \sum_{i=1}^{N} \log(1+e^{-d_i}) + \frac{d_i e^{-d_i/\tau}}{1+e^{-d_i/\tau}} \label{eq:x1}
\end{equation}
The margin $m_i$ for instance $i$ can be written as $m_i=
\begin{cases}
    d_i & x_i \text{ is class 1} \\
    -d_i & x_i \text{ is class 2}
\end{cases}
$.

\noindent The \eqref{eq:x1} can be written as
\begin{equation}
    \mathcal{L}_{t2l} = \sum_{y_i=0} \log(1+e^{-m_i}) + \frac{m_i e^{-m_i/\tau}}{1+e^{-m_i/\tau}} + 
    \sum_{y_j=1} \log(1+e^{m_j}) - \frac{m_j e^{m_j/\tau}}{1+e^{m_j/\tau}} \label{eq:x2}
\end{equation}
Let $m^* = \min(m_i)$ be the minimal margin, let $N_1$ and $N_2$ be the number of instances in class 1 and class 2 respectively which reaches the minimal margin. From the gradient analysis in \eqref{eq:gd}, the examples with $m_i > m^*$ has loss lower bounded by that with minimal margin. Then 

\begin{equation}
\begin{split}
    \mathcal{L}_{t2l} &= N_1(\log(1+e^{-m^*}) + \frac{m^* e^{-m^*/\tau}}{1+e^{-m^*/\tau}}) + 
    N_2(\log(1+e^{m^*}) - \frac{m^* e^{m^*/\tau}}{1+e^{m^*/\tau}}) \\
    & +O(\log(1+e^{-m^*}) + \frac{m^* e^{-m^*/\tau}}{1+e^{-m^*/\tau}}) + O(\log(1+e^{m^*}) - \frac{m^* e^{m^*/\tau}}{1+e^{m^*/\tau}})
\end{split}
\label{eq:x3}
\end{equation}

\noindent When $B$ approaches $\infty$, for $N_1$ part in \eqref{eq:x3},
\begin{equation}
\log(1+e^{-m^*}) + \frac{m^* e^{-m^*/\tau}}{1+e^{-m^*/\tau}} \sim e^{-m^*} + m^* e^{-m^*/\tau}
\end{equation}
When $m \rightarrow B$, $\lim_{m \rightarrow B}e^{-m^*} \rightarrow 0$, and $\lim_{m \rightarrow B}m^* e^{-m^*/\tau} = \lim_{m \rightarrow B}\frac{1}{1/\tau e^{m^*/\tau}} = 0$ by L'Hopital's rule.

\noindent For $N_2$ part in \eqref{eq:x3}, 
\begin{equation}
\log(1+e^{m^*}) - \frac{m^* e^{m^*/\tau}}{1+e^{m^*/\tau}} \sim \log(1+e^{m^*}) - m^* 
\end{equation}
When $m \rightarrow B$, $\lim_{m \rightarrow B} \log(1+e^{m^*}) - m^* = \lim_{m \rightarrow B}\log(1+\frac{1}{e^{m^*}}) = 0$.

\noindent Therefore, the loss is minimized when the minimal margin is maximized and thus the classifier converges to a max margin classifier when $B$ goes to infinity.
\end{proof}


%




\end{document}